%% file: acl_latex.tex
\documentclass[11pt]{article}

\usepackage[preprint]{acl}

\usepackage{times}
\usepackage{latexsym}
\usepackage{amsmath}
\usepackage{stfloats}
\usepackage{thmtools}
\usepackage{thm-restate}
\usepackage{amssymb}
\usepackage{float}
\usepackage[T1]{fontenc}
\usepackage[utf8]{inputenc}

\usepackage{microtype}

\usepackage{inconsolata}

\usepackage{graphicx}
\usepackage{hyperref}
\usepackage{url}
\usepackage{booktabs}
\usepackage{multirow}
\newcommand{\ours}{GIFT} % Define your algorithm name here
\usepackage[table]{xcolor}
\usepackage{graphicx} 
\usepackage{amsthm}
\usepackage{thmtools}
\usepackage{wrapfig}

\newtheorem{lemma}{Lemma}

\usepackage{tabularx}
\usepackage[ruled,vlined]{algorithm2e}
\title{\ours: Guided Importance-Aware Fine-Tuning for Diffusion Language Models}

\author{Guowei Xu$^\dagger$, Wenxin Xu$^\dagger$, Jiawang Zhao, Kaisheng Ma$^*$ \\
  Institute for Interdisciplinary Information Sciences \\
  Tsinghua University\\
 \textdagger Equal Contribution.\\
  *Corresponding Author: \texttt{kaisheng@tsinghua.edu.cn} }

\begin{document}
\maketitle
\begin{abstract}
Diffusion models have recently shown strong potential in language modeling, offering faster generation compared to traditional autoregressive approaches.
However, applying supervised fine-tuning (SFT) to diffusion models remains challenging, as they lack precise probability estimates at each denoising step.
While the diffusion mechanism enables the model to reason over entire sequences, it also makes the generation process less predictable and often inconsistent.
This highlights the importance of controlling key tokens that guide the direction of generation.
To address this issue, we propose \ours, an importance-aware finetuning method for diffusion language models, where tokens are assigned different importance weights based on their entropy.
Derived from diffusion theory, \ours\ delivers substantial gains: across diverse settings including different mainstream training datasets ranging from 1k to 10k in size, utilizing LoRA or full parameter fine-tuning, and training on base or instruct models, \ours\ consistently achieves superior overall performance compared to standard SFT on four widely used reasoning benchmarks (Sudoku, Countdown, GSM8K, and MATH-500).
\end{abstract}

\input{sections/intro}
\input{sections/preliminaries}

\input{sections/method}
\input{sections/experiment}
\input{sections/related}
\input{sections/conclusion}

\bibliography{custom}

\clearpage
\appendix
\input{sections/appendix}

\end{document}

%% file: sections/intro.tex
\section{Introduction}
Recently, the landscape of large language modeling has been reshaped by the emergence of diffusion-based approaches ~\citep{nie2025largelanguagediffusionmodels,ou2025absorbingdiscretediffusionsecretly,yang2025mmadamultimodallargediffusion}, which offer a fundamentally different perspective on sequence generation compared to autoregressive (AR) methods ~\cite{gpt1,gpt2,gpt3}. Rather than relying on strictly sequential decoding, diffusion large language models (dLLMs) employ iterative refinement procedures that enable efficient parallel generation and open new possibilities for scaling language models. 

Currently, similar to AR models, the training paradigm of dLLMs is generally divided into three stages: pre-training, instruction tuning, and reinforcement learning (RL). Among them, the pre-training and instruction tuning stages adopt an SFT loss specifically derived for dLLMs~\citep{nie2025largelanguagediffusionmodels}, while the RL stage employs a loss estimated on the basis of the PPO algorithm~\citep{schulman2017proximalpolicyoptimizationalgorithms}  and tailored to the characteristics of dLLMs ~\citep{zhu2025llada15variancereducedpreference}. 

Specifically, when designing the SFT objective, \citealp{nie2025largelanguagediffusionmodels} model the generation process of dLLMs as a continuous-time diffusion process, where the diffusion timestep $t$ ranges within $[0,1]$. Here, $t=0$ denotes a completely original, non-masked sequence, and $t=1$ denotes a fully masked sequence. As illustrated in Figure~\ref{fig:pipeline}(a), let the original response of the model be $x_0$. At a given timestep $t$, each token is masked with probability $t$, resulting in a partially masked response $x_t$. The training objective of dLLMs is then to enable the model to reconstruct the masked parts of $x_0$ given the partially masked input $x_t$. Based on this formulation, the loss is defined as follows:

\begin{equation}
L = - \sum_i \mathbb{E}_t \left[ \mathbf{1}[x_t^i = \mathbf{M}] \, \frac{1}{t} \log p\big( x_0^i \mid x_t \big) \right]
\end{equation}

This design of the SFT loss is intuitive and aligns well with the underlying rationale of diffusion theory. However, it presents several critical limitations. Among them, one of the most prominent issues is that the loss implicitly assumes a uniform masking rate across all tokens, thereby treating each token as equally important throughout the diffusion process. This assumption overlooks the inherent heterogeneity in token significance: tokens that play a central role in planning and reasoning should arguably receive greater emphasis during training than other tokens. While the Dream series ~\citep{ye2025dream7bdiffusionlarge} employs Context-Adaptive Token-Level Noise Rescheduling to account for varying token weights, its loss formulation remains heuristic and lacks a theoretical foundation, with its effectiveness proved limited in our experiments.

To address this discrepancy, we introduce \ours, a guided importance-aware finetuning mechanism based on predictive uncertainty.
Specifically, we find the entropy of the model's predictive distribution over tokens serves as a reliable metric for token importance at the current training stage.
Tokens with higher entropy correspond to positions where the model exhibits greater uncertainty in generation.
By prioritizing these high-entropy tokens during training via an adaptive importance weighting strategy applied to the SFT loss, we can concentrate computation and importance on tokens where the model is currently more uncertain, thereby improving the overall training effectiveness.

\begin{figure*}[t]
    \centering
    \includegraphics[width=\textwidth]{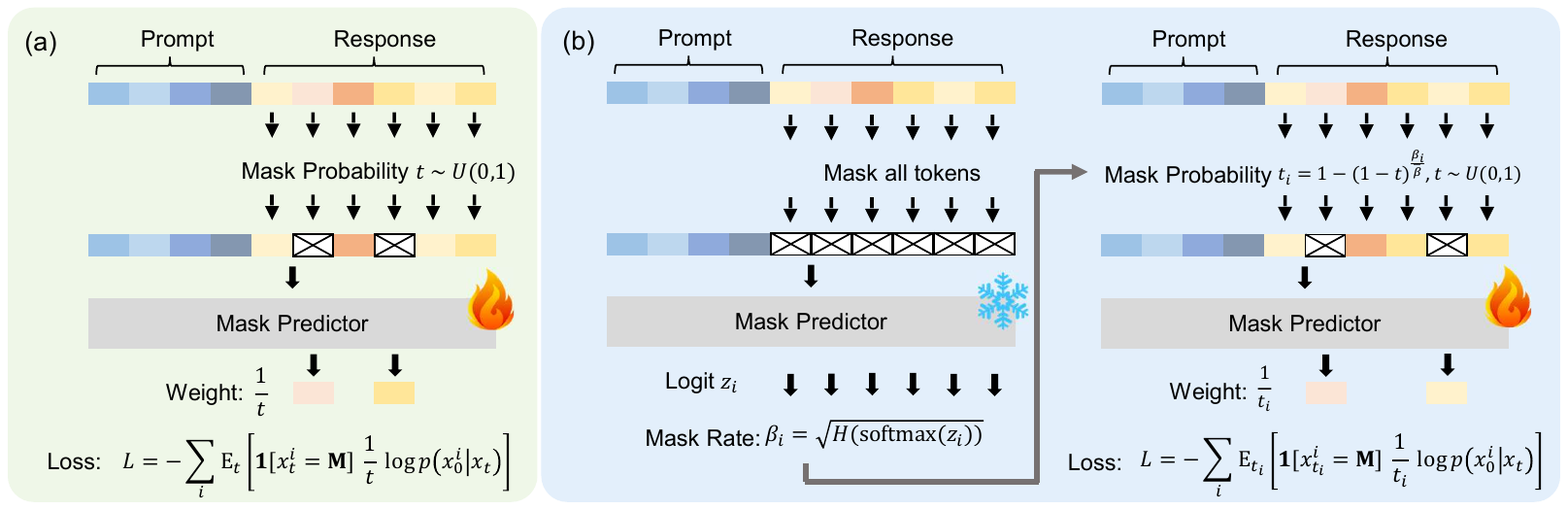} 
    \caption{(a) The SFT pipeline. A timestep $t$ is uniformly sampled from $[0,1]$, and each token is masked independently with probability $t$. The training objective is to predict the masked tokens based on the unmasked ones.
(b) The \ours\ pipeline. In each training step, we perform two forward passes. 
During the first forward pass, we mask the entire answer and estimate the masking rate $\beta_i$ for each token by computing its predictive entropy. 
In the second forward pass, the $i$-th token is masked with probability $t_i$ (computed from $\beta_i$), and its training weight is set to $\tfrac{1}{t_i}$. Tokens with higher entropy are more likely to be masked and thus receive stronger training signals.}
    \label{fig:pipeline}
\end{figure*}

Based on the fundamental principles of diffusion models, we derive an importance-weighted SFT loss formulation that is theoretically consistent with diffusion dynamics.
Empirically, we evaluate this approach on four widely used benchmarks, Sudoku~\citep{arel_sudoku}, Countdown~\citep{tinyzero}, GSM8K~\citep{cobbe2021trainingverifierssolvemath}, and MATH500~\citep{lightman2023letsverifystepstep}, by comparing models fine-tuned with the standard SFT algorithm against those fine-tuned with \ours\ under identical training data.
Specifically, our method is trained using LoRA \citep{hu2022lora} on s1K, s1K-1.1~\citep{muennighoff2025s}, and 3k samples drawn from the Mixture-of-Thoughts dataset of open-r1~\citep{openr1_mixture_of_thoughts_2025}, and using full parameter fine-tuning on 10k samples drawn from Tulu3~\citep{lambert2025tulu3pushingfrontiers}.
Across these diverse training settings, our approach consistently achieves superior overall performance compared to standard SFT.

In addition, ablation studies further confirm that both the theoretically derived importance-weighted SFT loss and the entropy-based importance-weighting scheme are indispensable for achieving the observed performance improvements. 
Time efficiency analysis shows that \ours\ incurs only minimal overhead during LoRA fine-tuning, and imposes no computational overhead during full-parameter fine-tuning, even achieving faster speeds than standard SFT.
Finally, experiments on few-shot learning and alternative model architectures verify the generalizability and superiority of \ours.

In summary, our core contributions are as follows:
\begin{itemize}
    \item We derive a theoretically grounded formulation of the importance-weighted SFT loss from the perspective of diffusion processes, ensuring consistency with the underlying principles of diffusion-based generation.
    \item We identify token-level predictive entropy as a reasonable and effective metric for capturing token importance, and we demonstrate both theoretically and empirically that tokens with higher entropy should be emphasized during training to enhance the effectiveness and stability of the training process.
    \item Building upon these insights, we propose \ours, an importance-weighted finetuning algorithm that consistently achieves superior overall performance compared to the standard SFT approach across multiple training sets and benchmarks. Ablation studies further validate the necessity and effectiveness of each design choice in \ours. Time efficiency analysis demonstrates that \ours\ incurs minimal training time overhead compared to SFT.
\end{itemize}

%% file: sections/preliminaries.tex
\section{Preliminaries: Continuous-Time Discrete Diffusion Model}

Following ~\citealp{ou2025absorbingdiscretediffusionsecretly}, we model diffusion language models as a \emph{Continuous-Time Discrete Diffusion Model}~\citep{sun2023scorebased}. Intuitively, in each infinitesimal time interval $\Delta t$, every token has a certain probability of being masked.  

Formally, this process can be described as a time-dependent Markov chain, whose transition dynamics are governed by a matrix $Q_t$. For a transition from state $x$ to $y$ within $(t, t+\Delta t)$, we have
\begin{equation}
p_{t+\Delta t|t}(y|x) = \left\{
\begin{aligned}
& Q_t(x,y)\Delta t +o(\Delta t), \\
& \quad \text{if } x \neq y \\
& 1+Q_t(x,x)\Delta t +o(\Delta t), \\
& \quad \text{if } x =y
\end{aligned}
\right.
\end{equation}
where $Q_t(x,y) \geq 0$ for $x \neq y$ and $Q_t(x,x) < 0$. Thus, $Q_t(x,y)$ specifies the instantaneous transition rate from state $x$ to $y$.  

In practice, it is very common to assume a \emph{time–factorized form}~\citep{campbell2022continuoustimeframeworkdiscrete}, 
\begin{equation}
Q_t = f(t)Q,
\end{equation}
where $Q$ is a constant matrix and $f(t)$ is a scalar function controlling the evolution speed.  

Let $P_{t|s}(x,y) = p_{t|s}(y|x)$ denote the transition probability matrix from time $s$ to $t$. Then, $P_{t|s}$ satisfies Kolmogorov's forward equation~\citep{anderson2012continuous}:
\begin{equation}
\frac{d}{dt} P_{t|s} = f(t) P_{t|s} Q.
\end{equation}

The solution is given by
\begin{equation}
P_{t|s} = \exp\!\left(Q \int_{s}^{t} f(u)\,du\right),
\end{equation}
where $\exp$ denotes the matrix exponential.  

For notational convenience, we define
\begin{equation}
\label{equ:f_bar}
\bar{f}(x) = \int_0^x f(t)\,dt.
\end{equation}
With this shorthand, the solution can be equivalently written as
\begin{equation}
P_{t|s} = \exp\!\big((\bar{f}(t) - \bar{f}(s))Q\big).
\end{equation}

Intuitively, $s$ corresponds to the sequence before masking at the current step, $t$ corresponds to the sequence after masking at the current step, and $Q$ represents the masking rate of each token.
In prior work, it is commonly assumed that the $Q$ matrix takes the following form: ~\citep{ou2025absorbingdiscretediffusionsecretly, nie2025largelanguagediffusionmodels} 

\begin{equation}
Q = \begin{bmatrix}
- 1 & 0 & \cdots & 0 &  1\\
0 & -1 & \cdots & 0 & 1 \\
\vdots & \vdots & \ddots & \vdots & \vdots \\
0 & 0 & \cdots & -1 & 1
\\
0 & 0 & \cdots & 0 & 0
\end{bmatrix}
\end{equation}

Under this assumption, all tokens gradually diffuse into an absorbing state [M], that is, they become masked.

%% file: sections/method.tex
\section{Method}
\subsection{Overview}

We propose \ours\ (Guided Importance-Aware Fine-Tuning), an importance-weighted finetuning method designed to assign higher importance weights to tokens where the current model exhibits greater uncertainty.
As we will explain later, in practice, we quantify token importance using the square root of token entropy.

As shown in Figure~\ref{fig:pipeline}(a), the original SFT approach masks all tokens with the same probability $t$, and each token contributes equally to the loss function with a weight of $\frac{1}{t}$.  In contrast, Figure~\ref{fig:pipeline}(b) illustrates the \ours \ pipeline. 
In each training step, we perform two forward passes. 
In the first pass, the entire answer is masked, and the masking rate $\beta_i$ for each token is estimated based on its predictive entropy.
In the second pass, each token is independently masked with probability $t_i$ (computed from $\beta_i$), and its training weight is assigned as $\tfrac{1}{t_i}$.
This design ensures that tokens with higher entropy are more likely to be masked and trained more frequently, thereby improving the effectiveness and stability of the training process.

The derivation of the loss function is presented in Section~\ref{sec:3.2}, while the details and motivation of the importance weighting scheme are provided in Section~\ref{sec:3.3} and Appendix~\ref{app:mot}.
A comprehensive description of the implementation of \ours\ can be found in Appendix~\ref{app:algo}.

\subsection{Importance-aware Loss Derivation}
\label{sec:3.2}

As noted earlier, prior work typically assumes that all tokens are masked at the same rate, leading to a $Q$ matrix that takes the form of an absorbing matrix composed of $1$ and $-1$. In contrast, we assign each token a distinct masking rate $\beta$, from which we derive an importance-weighted finetuning formulation. Importantly, the resulting SFT loss remains consistent with the fundamental properties of diffusion, since the $Q$ matrix is still well-defined under this construction.

Namely, we now set $Q$ as: 

\begin{equation}
\label{eq:q}
Q = \begin{bmatrix}
-\beta_1 & 0 & \cdots & 0 &  \beta_1 \\
0 & -\beta_2 & \cdots & 0 & \beta_2 \\
\vdots & \vdots & \ddots & \vdots & \vdots \\
0 & 0 & \cdots & -\beta_{n-1} & \beta_{n-1}
\\
0 & 0 & \cdots & 0 & 0
\end{bmatrix}
\end{equation}

Based on the newly defined form of the $Q$ matrix, and following the derivation in ~\citealp{ou2025absorbingdiscretediffusionsecretly}, we obtain the modified formulation of the SFT loss. The detailed derivation is provided in Appendix~\ref{app:der}.

\begin{restatable}{theorem}{maintheorem}
\label{thm:1}
Assuming the $Q$ matrix takes the form given in Equation~\ref{eq:q}, let the initial sequence be $x_0$ and the sequence at time $t$ be $x_t$. Under this setting, the $i$-th token is masked with probability $t_i = 1 - (1-t)^{\tfrac{\beta_{x^i}}{\beta_{\text{ref}}}}$, where $\beta_{x^i}$ denotes the masking rate of the $i$-th token, and $\beta_{\text{ref}}$ is a specified reference masking rate. Moreover, the importance-weighted supervised finetuning loss can be derived as follows:
\begin{equation}
    L =  - \sum_i \mathbb{E}_{t_i} \left[ \mathbf{1}\!\left[ x_{t_i}^i = \mathbf{M} \right] 
\frac{1}{t_i} \log p\big( x_0^i \mid x_t \big) \right]
\end{equation}

\end{restatable}

Intuitively, under the importance-weighted supervised finetuning loss, a token $x_i$ with a larger masking rate $\beta_{x^i}$ corresponds to a larger $t_i$, making it more likely to be masked and subsequently learned. Here, $\beta_{\mathrm{ref}}$ can be any reference value. Empirically, we set $\beta_{\mathrm{ref}}$ to the mean of all $\beta$ values  to enhance numerical stability.

\subsection{Motivation of Using Entropy as the Masking Rate Metric}
\label{sec:3.3}

\begin{figure*}[t]
    \centering
    \begin{minipage}{0.48\textwidth}
        \centering
        \includegraphics[width=\textwidth]{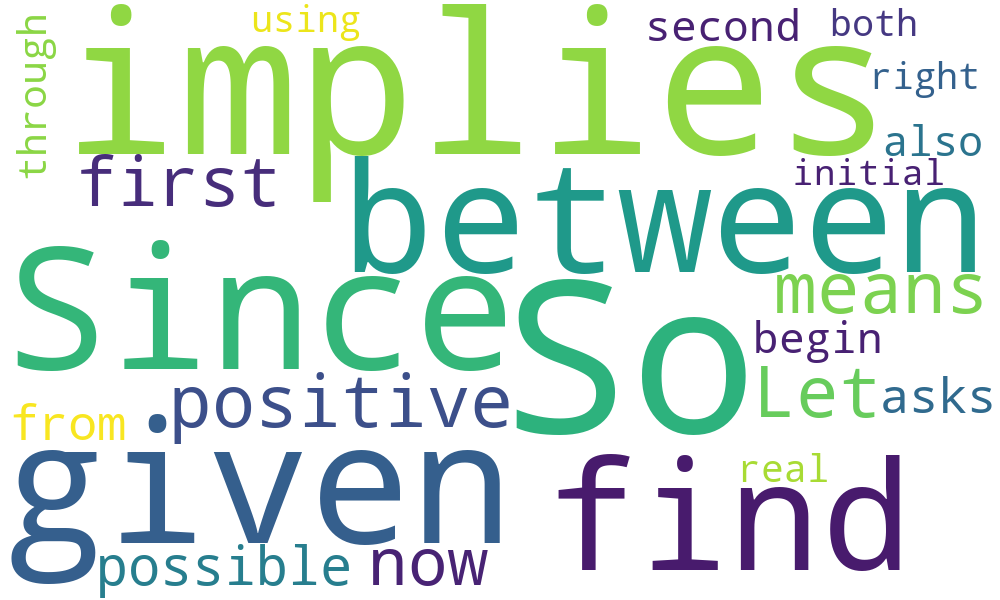}
        {Frequent tokens with high entropy}
    \end{minipage}
    \hfill
    \begin{minipage}{0.48\textwidth}
        \centering
        \includegraphics[width=\textwidth]{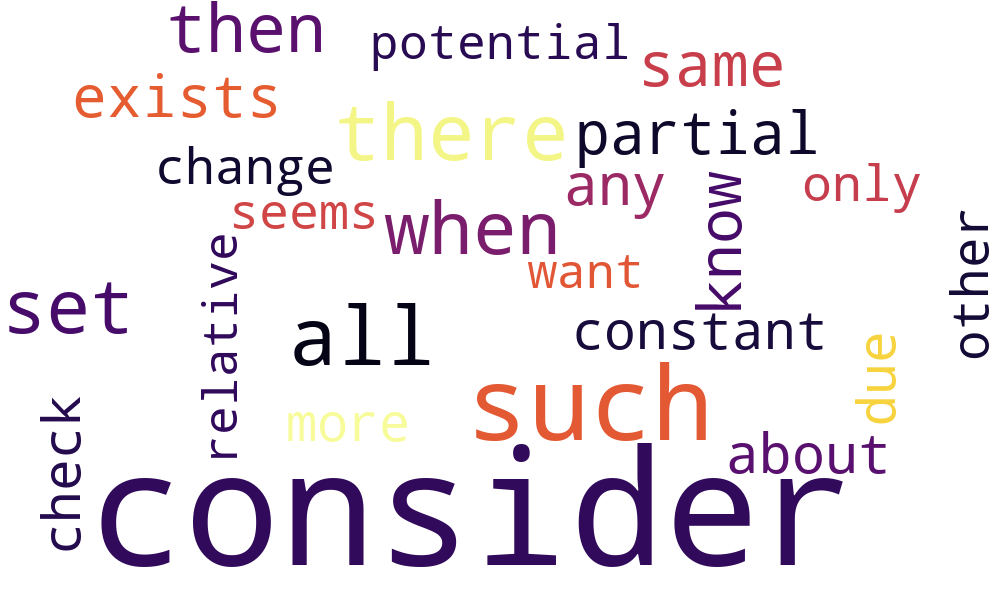}
    {Frequent tokens with low entropy}
    \end{minipage}
    \caption{Visualization of high-frequency tokens with different entropy levels. Tokens with higher entropy correspond to those where the model exhibits greater uncertainty, and are learned more frequently during training.}
    \label{fig:entropy}
\end{figure*}
Having established the form of the importance-weighted supervised finetuning loss, we next investigate how to assign masking rate to tokens.  
First, we require that the chosen metric does not introduce excessive computational overhead. In particular, it is infeasible to compute such a metric for every token individually, and we can only afford $O(1)$ additional forward passes in total. For this reason, we consider masking all the answers and then examining the distribution of the model’s output logits $z_i$ to be a practical choice. Specifically, given a problem $x$ and its answer $y$, we mask the entire answer and compute
\begin{equation}
z_i = \text{model}(\cdot \mid x, [M]).\text{logits}.
\end{equation}

We then determine the masking rate $\beta_i$ of the $i$-th token.
Based on our analysis, the entropy of the distribution of $z_i$ serves as an effective metric.
First, drawing inspiration from theoretical insights, we demonstrate that using entropy as an importance weight may help reduce the variance of the loss estimation.
Specifically, prior work has demonstrated that importance sampling achieves minimum variance when the sampling proposal is proportional to the L2 norm of the gradient. Through reasonable approximation, we found that token-wise entropy is approximately proportional to the L2 norm of the gradient, which motivates our choice of this metric.
We provide a detailed discussion in Appendix~\ref{app:mot}.
Furthermore, as shown in Figure~\ref{fig:entropy}, we empirically visualize the entropy distribution among the 100 most frequent tokens during training, contrasting tokens with higher versus lower entropy.
It can be observed that tokens with higher entropy tend to carry richer information and play a more critical role in the model's decision-making.
For example, words such as ``first'', ``second'', ``since'', and ``so'' exhibit higher entropy, often signaling logical or sequential structures, whereas words like ``all'' and ``such'' show lower entropy and function primarily as auxiliary words, contributing less to semantic content.
Our findings align with recent observations in LLMs~\citep{wang20258020rulehighentropyminority}, which note that a minority of high-entropy tokens serve to bridge logical connections between consecutive reasoning steps, while the majority of low-entropy tokens primarily complete current phrases or words. 

Since we empirically find that using the raw entropy values may destabilize training (Appendix ~\ref{exp:ablation_entropy}), we use the square root of entropy, which yields more stable optimization in practice:

\begin{equation}
\beta_i = \sqrt{\,H(\mathrm{softmax}(z_i))\,}.
\end{equation}

%% file: sections/experiment.tex
\section{Experiment}

In this section, we first describe the experimental setup, followed by a comparison of the performance of \ours \ and the standard SFT algorithm across multiple training datasets and benchmarks. We then conduct ablation studies to validate the necessity of each design component in \ours, and  provide a runtime comparison between \ours \ and SFT under identical settings. Finally, we provide a comparison between \ours \ and SFT under few-shot learning settings, as well as the generalizability of \ours \ across dLLM model architectures.

\begin{table*}[t]
\centering
\begin{tabular*}{\textwidth}{l@{\extracolsep{\fill}}ccccc}
\toprule
\multirow{2}{*}{\textbf{Model}} & 
\textbf{Sudoku}    & \textbf{Countdown} & \textbf{GSM8K} & \textbf{MATH500} & \multirow{2}{*}{\textbf{Average}} \\
& \textit{0-shot} & \textit{0-shot}   & \textit{0-shot} & \textit{0-shot} & \\
\midrule
\multicolumn{6}{l}{\textbf{Base Instruct Model}} \\
\midrule
LLaDA-8B-Instruct   &  5.6 $\pm$ 0.2 &  17.0 $\pm$ 1.0 & 76.8 $\pm$ 0.3  & 32.5 $\pm$ 0.3   &  33.0 $\pm$ 0.3 \\
\midrule
\multicolumn{6}{l}{\textbf{s1K Dataset, LoRA Tuning}} \\
\midrule
Instruct Model + SFT &  4.2 $\pm$ 0.4 &  21.1 $\pm$ 1.4 & \textbf{78.5 $\pm$ 0.2} & 32.6 $\pm$ 0.9 &  34.1 $\pm$ 0.4 \\
Instruct Model + \ours  &   \textbf{6.0 $\pm$ 0.9} & \textbf{28.1 $\pm$ 1.9} & 77.3 $\pm$ 0.3  & \textbf{33.7 $\pm$ 0.3} &  \textbf{36.3 $\pm$ 0.5} \\
\midrule
\multicolumn{6}{l}{\textbf{s1K-1.1 Dataset, LoRA Tuning}} \\
\midrule
Instruct Model  + SFT &  4.9 $\pm$ 0.3 &  20.7 $\pm$ 0.4  & \textbf{ 77.7 $\pm$ 0.3 } &  31.2 $\pm$ 0.8 &  33.6 $\pm$ 0.2 \\
Instruct Model + \ours  &   \textbf{5.9 $\pm$ 0.4} & \textbf{21.8 $\pm$ 0.8} & 77.6 $\pm$ 0.3  & \textbf{32.0 $\pm$ 0.9} &  \textbf{34.3 $\pm$ 0.3} \\
\midrule
\multicolumn{6}{l}{\textbf{openr1-3k Dataset, LoRA Tuning}} \\
\midrule
Instruct Model  + SFT &  5.9 $\pm$ 0.4 &  17.3 $\pm$ 0.8 & 77.1 $\pm$ 0.2  &  31.6 $\pm$ 0.6 &  33.0 $\pm$ 0.2 \\
Instruct Model  + \ours  &  \textbf{6.8 $\pm$ 0.3} & \textbf{18.8 $\pm$ 0.7} & \textbf{77.7 $\pm$ 0.3}  & \textbf{32.8 $\pm$ 0.8} & \textbf{34.0 $\pm$ 0.2} \\
\midrule
\multicolumn{6}{l}{\textbf{Tulu3-10k Dataset, Full Paramter Tuning}} \\
\midrule
 Base Model + SFT &  \textbf{5.1 $\pm$ 0.3} &  18.2 $\pm$ 0.9 & \textbf{66.3 $\pm$ 0.2} & 23.6 $\pm$ 0.6 & 28.3 $\pm$ 0.3 \\
 Base Model + \ours  &  4.9 $\pm$ 0.4  & \textbf{21.3 $\pm$ 1.1} & 65.8 $\pm$ 0.3 & \textbf{24.6 $\pm$ 0.8} & \textbf{29.2 $\pm$ 0.3} \\
\bottomrule
\end{tabular*}
\caption{Comparison of the performance of \ours \ and SFT across different training datasets and benchmarks.  In Appendix \ref{app:few}, we include performance metrics for 5-shot evaluation.}
\label{tab:main}
\end{table*}

\subsection{Experimental Setup} 

We adopt LLaDA-8B-Base and LLaDA-8B-Instruct~\citep{nie2025largelanguagediffusionmodels} as base models and apply LoRA fine-tuning~\citep{hu2022lora}  or full parameter fine-tuning.
For the training datasets, we use s1K, s1K-1.1, openr1-3k, and Tulu3-10k to evaluate the performance of \ours\ and SFT.
Here, s1K and s1K-1.1 are reasoning datasets proposed by~\citealp{muennighoff2025s}, while openr1-3k is constructed by sampling 3k examples from the Mixture-of-Thoughts dataset of open-r1~\citep{openr1_mixture_of_thoughts_2025}, with sequence lengths less than 2.5k under the open-r1 encoder.
Tulu3-10k consists of 10k samples drawn from Tulu3~\citep{lambert2025tulu3pushingfrontiers}, following the setting of~\citealp{dllm}.
To verify that our algorithm is framework-agnostic, we employ the \texttt{wd1} ~\citep{tang2025wd1weightedpolicyoptimization} framework for LoRA training, and the \texttt{dllm}~\citep{dllm} framework for full-parameter fine-tuning. All experiments are conducted across three independent runs, with results reported as mean value and uncertainty.

The evaluation is conducted on the following benchmarks:  

\begin{itemize}
    \item \textbf{Sudoku}~\citep{arel_sudoku}: A symbolic reasoning benchmark requiring models to solve Sudoku puzzles through  logical deduction.  
    \item \textbf{Countdown}~\citep{tinyzero}: A numerical reasoning dataset where models must combine given numbers using arithmetic operations to reach a target value.  
    \item \textbf{GSM8K}~\citep{cobbe2021trainingverifierssolvemath}: A widely used benchmark of grade-school math word problems designed to evaluate arithmetic and reasoning abilities.  
    \item \textbf{MATH500}~\citep{lightman2023letsverifystepstep}: A subset of the MATH dataset containing 500 challenging competition-level math problems that test advanced mathematical reasoning.  
\end{itemize}

The detailed training and evaluation hyperparameters and settings are provided in Appendix~\ref{app:hyper}.

\subsection{Main Results}
As shown in Table~\ref{tab:main}, we compare the performance of \ours\ and SFT across different training datasets and benchmarks.
After a single training run, we use 3 random seeds to evaluate model performance and report the average results with error bars.
Notably, \ours\ consistently achieves higher average accuracy than SFT, demonstrating that our method achieves effective training by leveraging token entropy. In Appendix \ref{app:few}, we include performance metrics for 5-shot evaluation.

\subsection{Ablation Study}
In this subsection, we conduct ablation studies to validate the design choices of \ours.  

\begin{table*}[t]
\centering
\begin{tabular*}{\textwidth}{l@{\extracolsep{\fill}}ccccc}
\toprule
\multirow{2}{*}{\textbf{Model}} & 
\textbf{Sudoku}    & \textbf{Countdown} & \textbf{GSM8K} & \textbf{MATH500} & \multirow{2}{*}{\textbf{Average}} \\
& \textit{0-shot} & \textit{0-shot}   & \textit{0-shot} & \textit{0-shot} & \\
\midrule
\textbf{Base Model} \\
\midrule
LLaDA-8B-Instruct   &  5.5 &  16.0 & 76.7  & 32.4   &  32.6 \\
\midrule
\textbf{s1K Dataset} \\
\midrule
+ SFT &  4.6 (-0.9) &  23.8 (+7.8) & 78.8 (+2.1) & 32.6 (+0.2) &  34.9 (+2.3)\\
+ Weighted SFT (- log p) & 3.0 (-2.5) & 16.4 (+0.4) & 75.4 (-1.3)& 29.2 (-3.2) & 31.0 (-1.6) \\
+ Dream & 3.3 (-2.2) & 21.1 (+5.1) & 77.0 (+0.3)& 30.0 (-2.4) & 32.8 (+0.2) \\
+ \ours \ (SW) & 5.3 (-0.2) & 23.8 (+7.8) & 75.9 (-0.8)& 31.6 (-0.8) & 34.2 (+1.6) \\
+ \ours  &   7.8 (+2.3) & 24.6 (+8.6) & 78.0 (+1.3)  & 33.0  (+0.6) &  35.8 (+3.2)\\
\bottomrule
\end{tabular*}
\caption{Ablation Study. (a) Comparison between entropy-based weighting and NLL-based weighting. As shown in the table, Weighted SFT (SFT algorithm with NLL-based importance weighting) performs significantly worse. (b) Comparison between \ours, \ours \ (SW) (token-wise weighting using our square-root entropy), and Dream. 
Dream yields negligible improvements. 
\ours \ (SW) provides some gains, but still performs worse than SFT.}
\label{tab:abl}
\end{table*}

\textbf{Necessity of Using Entropy as the Metric.}
To demonstrate the necessity of entropy-based importance weighting, we compared our method with a baseline using the negative log-likelihood (NLL). Unlike entropy-based importance weighting, NLL emphasizes tokens that appear less frequently rather than those that are difficult to predict or carry more information. 
Assuming the logits are $z_i$, the masking rate $\beta_i$ for the $i$-th token under the NLL setting is:

\begin{equation}
    \beta_i = -\log\bigl(\text{softmax}(z_i)^{x_0^i}\bigr) 
\end{equation}

As shown in Table~\ref{tab:abl}, the SFT method weighted by NLL performs poorly, in some cases even worse than the original base model. This underscores the importance of focusing on high-entropy, informative tokens for achieving effective fine-tuning.

\textbf{Validity of the Theoretically Derived Loss Function.}
Our loss function is derived from the continuous-time discrete diffusion model combined with the Denoising Score Entropy (DSE) loss. As a result, our model preserves the inherent characteristics of the diffusion process more effectively than methods that simply add a weight to the loss function:

\begin{equation}
L = -\sum_i \mathbb{E}_t \left[ \mathbf{1}\!\left[x_t^i = \mathbf{M}\right] 
\frac{w_i}{t} \log p\left(x_0^i \mid x_t\right) \right]
\end{equation}

To illustrate this, we compare our approach with two weighted SFT methods by applying a simple token-wise weighting, namely using our square-root entropy as the weight, and using the Dream~\citep{ye2025dream7bdiffusionlarge} method. 
Specifically, when using our square-root entropy, the weight of the $i$-th token is:  

\begin{equation}
w_i = \sqrt{\,H(\mathrm{softmax}(z_i))\,}
\end{equation}

When using the weighting method from Dream, the weight of the $i$-th token is:  

\begin{equation}
w_i = \tfrac{1}{2} \sum_{j=1}^N 
\mathbf{1}\!\left[x_t^j \neq \mathbf{M}\right] 
\, \mathrm{Geo}\!\left(p, \lvert j - i \rvert - 1\right)
\end{equation}

Here, Geo denotes the geometric distribution, and $p$ controls the sharpness of the distribution. Since the original work did not provide the corresponding parameter setting, we set it to a moderate value of 0.3.

Table~\ref{tab:abl} presents a comparison between \ours, \ours \ (SW) (i.e., simple token-wise weighting using our square-root entropy as the weight), and Dream. 
Dream yields only minimal improvements and, in some cases, performs worse than standard SFT. 
\ours \ (SW) provides some gains but still performs worse than SFT, whereas \ours \ significantly outperforms SFT. 
This demonstrates that the diffusion loss derived from theoretical principles possesses superior properties.

\begin{table*}[t]
\centering
\begin{tabular*}{\textwidth}{l@{\extracolsep{\fill}}ccccc}
\toprule
\multirow{2}{*}{\textbf{Model}} & 
\textbf{Sudoku}    & \textbf{Countdown} & \textbf{GSM8K} & \textbf{MATH500} & \multirow{2}{*}{\textbf{Average}} \\
& \textit{0-shot} & \textit{0-shot}   & \textit{0-shot} & \textit{0-shot} & \\
\midrule
\multicolumn{6}{l}{\textbf{Tulu3-10k Dataset, Full Paramter Tuning}} \\
\midrule
Base Model + SFT & \textbf{3.8 $\pm$ 0.4} & $12.3 \pm 0.7$ & $50.3 \pm 0.3$ & $17.6 \pm 0.6$ & $21.0 \pm 0.3$ \\
Base Model + \ours  & $2.1 \pm 0.3$ & \textbf{14.1 $\pm$ 0.8} & \textbf{55.2 $\pm$ 0.2} & \textbf{22.2 $\pm$ 0.5} & \textbf{23.4 $\pm$ 0.3} \\
\bottomrule
\end{tabular*}
\caption{Full-parameter fine-tuning results on Dream-v0-7B-Base.}
\label{tab:dream_generalization}
\end{table*}

\subsection{Time Efficiency Analysis}
In the time efficiency analysis, there are two counteracting factors.
First, \ours\ requires an additional forward pass compared to SFT, resulting in a slightly higher computational overhead.
However, it is important to note that the cost of a single gradient-free forward pass is significantly lower than that of a backward pass; thus, this factor only causes a marginal decrease in speed.
Second, the expected number of masked tokens per iteration may differ between \ours\ and SFT, which also impacts the overall training speed.

\begin{table}[H]
    \centering
    \begin{tabular*}{0.9\linewidth}{l@{\extracolsep{\fill}}ccc}
        \toprule
        \textbf{Dataset} & \textbf{Framework} & \textbf{SFT} & \textbf{\ours} \\
        \midrule
        s1K & wd1 & 46 min & 57 min \\
        Tulu3 & dllm &28 min & 26 min \\
        \bottomrule
    \end{tabular*}
    \caption{Time efficiency comparison between SFT and \ours\ on different tasks and frameworks.}
    \label{tab:speed_comparison}
\end{table}

As shown in Table~\ref{tab:speed_comparison}, we compare the training speed of \ours\ and SFT under two distinct settings:
(1) LoRA fine-tuning on the \texttt{s1K} dataset for 20 epochs using the \texttt{wd1} training framework on 4 H100 GPUs; and
(2) full fine-tuning on the \texttt{Tulu3-10k} dataset for 5 epochs using the \texttt{dllm} training framework on 8 H100 GPUs.
These scenarios represent representative results selected from our main experiments.
The results indicate that while \ours\ is merely slightly slower than SFT during LoRA fine-tuning, it is marginally faster than SFT during full fine-tuning.
We attribute this to the fact that the additional forward pass incurs only a minor overhead.
Conversely, since \ours\ prioritizes more significant tokens, the expected total number of masked tokens may be lower than that of SFT, thereby reducing the computational load during backpropagation.
This reduction in computation has a more pronounced impact during full fine-tuning compared to LoRA, leading to the observation that our method achieves faster speeds than SFT in the full fine-tuning setting.

\subsection{Generalizability across Architectures}
To verify that our method is applicable to dLLMs of various architectures and to establish generalizability beyond the LLaDA-8B backbone, we perform full-parameter fine-tuning on Dream-v0-7B-Base using the dLLM framework and the Tulu3-10k dataset. 

In Table \ref{tab:dream_generalization}, the results further demonstrate the superiority of the GIFT algorithm over standard SFT and verify the generalizability of our approach across different model architectures.

%% file: sections/related.tex
\section{Related Work}

\textbf{Diffusion Language Models.}
Diffusion models, initially developed for continuous data, have recently been adapted to discrete sequences, including natural language. Early work such as Structured Denoising Diffusion Models (D3PMs) \citep{austin2023structureddenoisingdiffusionmodels} extended discrete diffusion to text, showing competitive performance in character-level generation and large vocabulary language modeling. DiffusionBERT \citep{he2022diffusionbertimprovinggenerativemasked} combined masked language modeling with diffusion, improving generative quality over conventional masked LMs. Later, masked diffusion frameworks \citep{sahoo2024simpleeffectivemaskeddiffusion, shi2025simplifiedgeneralizedmaskeddiffusion} introduced simplified and generalized training objectives that reduce complexity while maintaining performance, and support state-dependent masking strategies. Diffusion large language models (dLLMs) including LLaDA \citep{nie2025largelanguagediffusionmodels}  and Dream ~\citep{ye2025dream7bdiffusionlarge} demonstrated that diffusion-based models can scale to billions of parameters, achieving strong context learning and instruction-following capabilities. Techniques such as variance-reduced preference optimization \citep{zhu2025llada15variancereducedpreference} and reinforcement of lateral thought chains \citep{huang2025reinforcingdiffusionchainlateral} further enhance generative stability and reasoning ability, while efficient perplexity bounds and ratio matching \citep{haxholli2025efficientperplexityboundratio} improve sampling efficiency. More recently, works such as d1~\citep{zhao2025d1scalingreasoningdiffusion} and wd1~\citep{tang2025wd1weightedpolicyoptimization} have explored the possibility of applying reinforcement learning on top of existing dLLMs. LLaDA 2.0 \citep{bie2025llada20scalingdiffusionlanguage} pushed the boundaries of model capacity by scaling diffusion language models to 100B parameters, offering a comprehensive analysis of scaling laws.

\textbf{Weighted Methods in Language Model Learning.}
A complementary line of research focuses on improving large language models through weighted supervised fine-tuning (SFT) and reinforcement learning. 
Weighted SFT techniques aim to prioritize informative tokens. 
DFT~\citep{wu2025generalizationsftreinforcementlearning} addresses the limited generalization of SFT by dynamically rescaling token-level objectives to stabilize gradient updates. 
Rho-1~\citep{lin2025rho1tokensneed} challenges the conventional uniform next-token prediction objective by selectively training on useful tokens identified via a reference model. 
Similarly, \citealp{luo2023reweightingtokenssimpleeffective} propose a re-weighting strategy for active learning in NER, assigning dynamic smoothed weights to tokens to  boost performance.
In reinforcement learning, applying weighting to improve learning effectiveness is also a common strategy. 
\citealp{cui2025entropymechanismreinforcementlearning} investigate the entropy dynamics of reinforcement learning for reasoning LLMs, and propose entropy-based weighting methods to prevent entropy collapse and encourage exploration. 
\citealp{wang20258020rulehighentropyminority} highlight the critical role of high-entropy minority tokens in RLVR, demonstrating that restricting updates to these tokens significantly enhances reasoning capabilities.  
Together, these approaches demonstrate that token-level weighting is an important and effective way to improve the optimization of language models.

Despite these advancements, most weighting methods explored for classical LLMs cannot be directly transferred to the dLLM setting. For instance, methods like DFT and Rho-1 rely on the specific properties of auto-regressive models. While the work by \citealp{wang20258020rulehighentropyminority} is related to ours in emphasizing the importance of high-entropy tokens, they merely apply a hard threshold to select tokens for gradient computation. In contrast, our \ours\ weighting method is tailored to the unique characteristics of dLLMs. It determines importance weights for all tokens through a single forward pass, making our approach fundamentally distinct from these previous weighting mechanisms.

%% file: sections/conclusion.tex
\section{Conclusion}
We introduced \ours, a theoretically grounded importance-aware fine-tuning algorithm for diffusion language models.
By leveraging token-level entropy as a measure of importance, \ours\ prioritizes high-uncertainty tokens and preserves diffusion process properties.
Experiments across varying datasets, different training frameworks, multiple benchmarks, and on both base and instruct models demonstrate the universal advantage of the \ours\ method over SFT.
These results highlight entropy-based importance weighting as a simple yet powerful strategy for enhancing training effectiveness and stability in diffusion language models.

\clearpage

%% file: sections/appendix.tex
\section{Proof of the Weighted SFT Loss}
\label{app:der}

In order to prove this theorem, we will first prove two lemmas, and then proceed to prove the main theorem. Our proof follows that of~\citealp{ou2025absorbingdiscretediffusionsecretly}, with the key difference that our $Q$ matrix incorporates varying masking rates $\beta$, whereas~\citealp{ou2025absorbingdiscretediffusionsecretly} assumes all masking rates are identical.

Here, we assume a fixed matrix $Q$ during the derivation; however, in practical entropy-driven masking, $Q$ functions more as a heuristic variance-control scheme. In Appendix~\ref{app:rigor}, we provide a fully theory-compliant version of \ours\ and find that it also outperforms SFT.

\subsection{Proof of Lemmas}
Let the original response of the model be $x_0$. At a given timestep $t$, 
let the partially masked response be $x_t$, and $[M]$ denotes the mask. 
$\beta_{x_t^i}$ denotes the masking rate corresponding to the $i$-th token. 
The definition of $\bar{f}$ can be found in Equation~\ref{equ:f_bar}.

\begin{lemma}\label{lem:masking}
Let $x_t^i = [M]$ and ${x'}_t^i \neq [M]$. Then we have
\begin{align}
    &\frac{p_t(x_t^1,\dots,{x'}^i_t,\dots,x_t^d)}{p_t(x_t^1,\dots,x^i_t,\dots,x_t^d)} \nonumber \\
    &\quad = \frac{\exp(-\beta_{x_t^i}\,\bar{f}(t))}{1-\exp(-\beta_{x_t^i}\,\bar{f}(t))}\; p_0({x'}_t^i \mid x_t^{\mathrm{unmasked}}).
\end{align}
On the other hand, if $x_t^i \neq [M]$, then
\begin{equation}
\frac{p_t(x_t^1,\dots,{x'}^i_t,\dots,x_t^d)}{p_t(x_t^1,\dots,x^i_t,\dots,x_t^d)} = 0;
\end{equation}
\end{lemma}

Intuitively, this theorem characterizes the exact transition probabilities during the training process. It also establishes a connection between the concrete score function and the clean data distribution via a scalar factor.

\begin{proof}

We first calculate $p_{t|0}(x_t|x_0)$:

Based on the definition of $Q_t$, we can easily see that $p_{t|0}(x_t|x_0)=0$ if $x_t\neq x_0$ and $x_t\neq [M]$

Now we divide the time interval $(0,t)$ into $n$ intervals, namely $(0,s_1),(s_1,s_2),\cdots,(s_{n-1},t)$. Here we denote $s_0=0,s_n=t$

then the total probability of $p_{t|0}(x_t=x_0|x_0)$ is 
\begin{align*}
    &\prod_{i=0}^{n-1}p_{s_{i+1}|s_i}(x_{s_{i+1}}=x_{s_i}|x_{s_i}) \\ 
    &=\prod_{i=0}^{n-1}\Big(1+Q_{s_i}(x_{s_i},x_{s_i})(s_{i+1}-s_i) \nonumber \\
    &\quad +o(s_{i+1}-s_i)\Big) \\ 
    &=\exp\Big(\sum_{i=0}^{n-1}\ln\Big(1+Q_{s_i}(x_{s_i},x_{s_i}) \nonumber \\
    &\quad \times (s_{i+1}-s_i) +o(s_{i+1}-s_i)\Big)\Big) \\
    &=\exp\Big(\sum_{i=0}^{n-1}[f(s_i)Q(x_{s_i},x_{s_i})(s_{i+1}-s_i) \nonumber \\
    &\quad +o(s_{i+1}-s_i)]\Big)
\end{align*}
Now we select a series of  appropriate $n$ and $s_1,s_2,\dots,s_{n-1}$ such that while $n\rightarrow +\infty$ , $\max(s_{i+1}-s_i)\rightarrow 0$. Now by recalling the definition of integral, we obtain: 
\begin{align*}
    &\exp\Big(\sum_{i=0}^{n-1}[f(s_i)Q(x_{s_i},x_{s_i})(s_{i+1}-s_i) \nonumber \\
    &\quad +o(s_{i+1}-s_i)]\Big) \rightarrow \exp(-\beta_{x_0}\bar{f}(t))
\end{align*}

Thus, we have \[
p_{t|0}(x_t|x_0) = \begin{cases}
\exp(-\beta_{x_0}\bar{f}(t)) \\
\quad \text{if } x_t = x_0 \\
1-\exp(-\beta_{x_0}\bar{f}(t)) \\
\quad \text{if } x_t = [M] \\
0 \\
\quad \text{otherwise}
\end{cases}
\]

Now, based on the above result, recalling the independence of different dimensions in the diffusion process, we can get a stronger result: 

\begin{lemma}Suppose that we have $x_t$ in which the tokens $a_1,a_2,...,a_k$ are masked 
and $b_1,b_2,...,b_{n-k}$ are unmasked; then:
\begin{align}
p_t(x_t) &= \prod_{i=1}^k(1-\exp(-\beta_{a_i}\bar{f}(t))) \nonumber \\
&\quad \times \prod_{j=1}^{n-k}\exp(-\beta_{b_j}\bar{f}(t))p_0(x_t^{\text{unmasked}})
\end{align}
\end{lemma}

\begin{proof}
Exploiting the independence across dimensions in the diffusion process, we may, without loss of generality, set $a_i = i$ and $b_i = k+i$. Equivalently, this corresponds to masking the first $k$ tokens while leaving the remaining $n-k$ tokens unmasked.

Then we obtain 
\begin{align*} 
&p_t([M]\dots[M]x_t^{k+1}\dots x_t^n)\\ 
&=\sum_{x_0}p_{t|0}([M]\dots[M]x_t^{k+1}\dots x_t^n|x_0)p_0(x_0)\\
&=\sum_{x_0^1\dots x_0^n}\prod_{i=1}^k p^k_{t|0}([M]|x_0^{k})\nonumber \\ &\quad \times\prod_{i=k+1}^np_{t|0}( x_t^k|x_0^k)p_0(x_0^1,\dots x_0^n)\\
&=\sum_{x_0^1\dots x_0^n}\prod_{i=1}^k (1-\exp(-\beta_i\bar{f}(t))) \nonumber \\
&\quad \times \prod_{i=k+1}^n\exp(-\beta_i\bar{f}(t))p_0(x_0^1,\dots x_0^n)\\
&=\sum_{x_0^{k+1}\dots x_0^n}\prod_{i=1}^k (1-\exp(-\beta_i\bar{f}(t))) \nonumber \\
&\quad \times \prod_{i=k+1}^n\exp(-\beta_i\bar{f}(t))p_0(x_0^{k+1},\dots x_0^n)
\end{align*}
Recalling the definition, $p_0(x_0^{k+1}, \dots, x_0^n)$ can be expressed as $p_0(x_t^{\text{unmasked}}),$ and the desired result follows.

\end{proof}

Now we append $x'_t$ and $x_t$ to the above result and we obtain the result of Lemma~\ref{lem:masking}.

\end{proof}

\subsection{Proof of the Main Theorem}

\maintheorem* 

\begin{proof}
Now we consider the Denoising Score Entropy, introduced in \cite{lou2024discretediffusionmodelingestimating}. DSE  loss provides a principled training objective for diffusion models by measuring the discrepancy between the model score function and the true score of the perturbed data distribution.

The DSE loss can be written as:

\begin{align*}
&\int_{0}^T\mathbb{E}_{x_t\sim p_{t|0}(x_t|x_0)} \nonumber \\
&\quad \times \sum_{x_t'\neq x_t}Q_t(x_t',x_t)\Big(s_{\theta}(x,t)_{x_t'} \nonumber \\
&\quad -\frac{p_{t|0}(x_t'|x_0)}{p_{t|0}(x_t|x_0)}\log s_{\theta}(x_t,t)_{x_t'}+C\Big)dt
\end{align*}

\[
C = K\!\left(\frac{p_{t|0}(x_t'|x_0)}{p_{t|0}(x_t|x_0)}\right), 
\quad K(a) = a \log a - a,
\]
Here $C$ is a constant irrelevant to the loss.  
The function $s_{\theta}(x_t,t)$ denotes a score network used to estimate the ratio$\frac{p_t(x_t')}{p_t(x_t)}.$

In practice, we often estimate $s_{\theta}(x,t)_{x_t'}$ using 
\begin{align*}
s_{\theta}(x,t)_{x_t'} &\approx p_t(x_t^1\dots {x'}^i_t \dots x_t^d) \nonumber \\
&\quad \times \frac{1}{p_t(x_t^1\dots x^i_t\dots x_t^d)}
\end{align*}

We denote $\tilde{Q}$ as the reverse transition matrix corresponding to $Q$. Then, the reverse process satisfies \[
\tilde{Q}_t(x_t,x_t') = \begin{cases}
\frac{p_t(x_t')}{p_t(x_t)} Q_t(x_t',x_t) 
\ \  \text{if } x_t' \neq x_t \\
-\sum_{k\neq x_t}\tilde{Q_t}(x_t,k)
 \ \ \text{if } x_t'=x_t
\end{cases}
\]

Now we use Lemma~\ref{lem:masking} and calculate the three parts of the loss 
separately:

Using the  fact that $$\sum p_0({x'}_t^i|x_t^{\text{unmasked}})=1$$ we obtain:

\begin{align*}
&\int_{0}^T\mathbb{E}_{x_t\sim p_{t|0}(x_t|x_0)}\sum_{x_t'\neq x_t}Q_t(x_t',x_t)s_{\theta}(x,t)_{x_t'}dt \\
&=\int_{0}^T\mathbb{E}_{x_t\sim p_{t|0}(x_t|x_0)}
\sum_{x_t^i=[M]}\beta_if(t) \nonumber \\
&\quad \times \frac{\exp(-\beta_{x_t^i}\bar{f}(t))}{1-\exp(-\beta_{x_t^i}\bar{f}(t))}dt
\end{align*}

Using the properties of $-\frac{p_{t|0}(x_t'|x_0)}{p_{t|0}(x_t|x_0)}$ we obtain:

\begin{align*}
& \int_{0}^T\mathbb{E}_{x_t\sim p_{t|0}(x_t|x_0)}\sum_{x_t^i\neq x_t}Q_t(x_t',x_t) \\
&\quad \times \left(-\frac{p_{t|0}(x_t'|x_0)}{p_{t|0}(x_t|x_0)}\log s_{\theta}(x_t,t)_{x_t'}\right)dt\\ 
&= \int_{0}^T\mathbb{E}_{x_t\sim p_{t|0}(x_t|x_0)}\sum_{x_t^i=[M],j\neq [M]}-f(t)\beta_i \\
&\quad \times \mathbf{1}[x_0^i=j]\log s_{\theta}(x_t,t)_{x_t'}dt 
\end{align*}

Recalling that $K(a) = a \log a - a$ with $K(0) = 0$, 
we can safely ignore the terms $\mathbf{1}[x_0^i = j]$ and $j \neq [M]$.

\begin{align*}
&\int_{0}^T \mathbb{E}_{x_t\sim p_{t|0}(x_t\mid x_0)} \sum_{x_t'\neq x_t} Q_t(x_t',x_t) \nonumber \\
&\quad \times K\!\left(\frac{p_{t|0}(x_t'\mid x_0)}{p_{t|0}(x_t\mid x_0)}\right)\, dt \\
&= \int_{0}^T
\mathbb{E}_{x_t\sim p_{t|0}(x_t\mid x_0)}
\sum_{x_t^i=[M],j\neq [M]}\beta_i f(t) \\
&\quad \times
K\!\left(
\frac{\exp(-\beta_{x_t^i}\bar{f}(t))}
     {1-\exp(-\beta_{x_t^i})\bar{f}(t)}
\mathbf{1}[x_0^i=j]
\right) dt \\
&= \int_{0}^T
\mathbb{E}_{x_t\sim p_{t|0}(x_t\mid x_0)}
\sum_{x_t^i=[M]}\beta_i f(t) \nonumber \\
&\quad \times \frac{\exp(-\beta_{x_t^i}\bar{f}(t))}{1-\exp(-\beta_{x_t^i}\bar{f}(t))} \nonumber \\
&\quad \times \left(\log \frac{\exp(-\beta_{x_t^i}\bar{f}(t))}{1-\exp(-\beta_{x_t^i}\bar{f}(t))}-1\right)dt 
\end{align*}

By adding them all together, especially combining the first and third term, we get

\begin{align*}
&\int_{0}^T\mathbb{E}_{x_t\sim p_{t|0}(x_t|x_0)}\sum_{x_t^i=[M]}\frac{-\beta_if(t)\exp(-\beta_i\bar{f}(t))}{1-\exp(-\beta_i\bar{f}(t))} \\
&\times \log\left(\frac{\exp(-\beta_i\bar{f}(t))}{1-\exp(-\beta_i\bar{f}(t))}q_{\theta}(x_0^i|x_t^{\text{unmasked}})\right) dt
\end{align*}

Now we apply a change of variables. By defining
\[
t_i = 1 - \exp\!\big(-\beta_{x_t^i} \, \bar{f}(t)\big),\]

we obtain 
\[dt_i = \beta_{x_t^i}f(t)\exp\!\big(-\beta_{x_t^i} \, \bar{f}(t)\big)dt
\]
and the loss can be rewritten as:

\begin{align*}
&\sum_{i}\int_0^1\mathbb{E}_{x\sim p_{t|0}(x_t|x_0)}\Big[-\frac{1}{t_i}\mathbf{1}[x_t^i=[M]] \nonumber \\
&\quad \times \log(q_{\theta}(x_0^i|x_t^{\text{unmasked}}))\Big]dt_i+C
\end{align*}

Here, $x_{t_i}$ denotes the distribution in which $x^i$ is masked with probability $t_i$. 

Recall that 
\[
t_i = 1 - \exp(-\beta_{x^i} \, \bar{f}(t)).
\]

Suppose we have a reference value $\beta_{\text{ref}}$ and define 
\[
t = 1 - \exp(-\beta_{\text{ref}} \, \bar{f}(t)).
\] 
Then we can express $t_i$ as

\begin{align*}
t_i &=  1 - \exp(-\beta_{x^i} \, \bar{f}(t)) \\
&=1-\exp(-\beta_{x^i}(\frac{1}{\beta_{ref}}\log(1-t)))\\
&=1 - (1-t)^{\frac{\beta_{x^i}}{\beta_{\text{ref}}}}
\end{align*}

In practice, one can sample $t$ and calculate $t_i$ using the above formula. 
\end{proof}

\section{Algorithm Pipeline}
\label{app:algo}
In Algorithm~\ref{alg:ours_part1} and Algorithm~\ref{alg:ours_part2}, we provide a detailed description of the implementation of \ours.

\begin{algorithm}[]
\caption{Creating Betas from Entropy}
\label{alg:ours_part1}

\KwIn{model $\mathcal{M}$, input\_ids $\mathcal{I}$,  mask\_token\_id $m$}
\KwOut{Masking parameter $\beta$}

\SetKwFunction{EstimateBetas}{EstimateRate}

\BlankLine
\textbf{Function} \EstimateBetas{$\mathcal{M}, \mathcal{I},  M, m$}: \\
\Indp
\tcp{mask out the answer}
    \For{$i \gets 1$ \KwTo $|\mathcal{I}|$}{
    $\mathcal{I'}^i \gets (1 - M^i) \cdot \mathcal{I}^i + M^i \cdot m$\; 
    } 
    $z \leftarrow \mathcal{M}(\mathcal{I'}).\text{logits}$\;
    \tcp{calculate entropy}
     \For{$i \gets 1$ \KwTo $|\mathcal{I}|$}{
    $H_i \leftarrow -\sum_j \text{softmax}(z_{ij}) \cdot \text{log\_softmax}(z_{ij})$\;  
    $\beta_i \gets M^i\sqrt{H_i}$\;
    }
    \Return $\beta$\;
\Indm
\end{algorithm}

\clearpage
\begin{algorithm}[t!]
\caption{Weighted Entropy-driven Fine-Tuning}
\label{alg:ours_part2}

\SetKwFunction{EstimateBetas}{EstimateRate}
\SetKwFunction{ComputeLoss}{ComputeLoss}
\KwIn{model $\mathcal{M}$, input\_ids $\mathcal{I}$, labels $l$, mask\_token\_id $m$, prompt lengths $L$}
\KwOut{Loss $\mathcal{L}$}

\BlankLine
\textbf{Function} \ComputeLoss{$\mathcal{M},\mathcal{I}, m, L$}: \\
\Indp
    $t \sim \mathcal{U}(0,1)$\;
    \tcp{mask the model's answer}
    \For{$i \gets 1$ \KwTo $|\mathcal{I}|$}{
    $M^i = \mathbf{1}_{\{\,i > L\,\}}$ 
}
    $\beta \leftarrow \EstimateBetas(\mathcal{M}, \mathcal{I}, M,  m)$\;
    $\beta_{ref} \leftarrow \frac{\sum_i \beta_i \cdot M_i}{\sum_i M_i}$ \;

    \tcp{Compute the masking probability for each token}
     \For{$i \gets 1$ \KwTo $|\mathcal{I}|$}{
    $t_i\leftarrow 1-(1-t)^{\frac{\beta_i}{\beta_{ref}}}$\;
    $M'^i \sim \text{Bernoulli}(t_i)$ \;
    }
    $\mathcal{I'} \leftarrow (1 - M') \cdot \mathcal{I} + M' \cdot m$\;
    $z \leftarrow \mathcal{M}(\mathcal{I'}).\text{logits}$\;
    $S \gets \sum_{i} \mathbf{1}_{\{M'^i=1\}}\frac{1}{t_i} \cdot$   
    \Indp $\text{CrossEntropy}(z_i, l_i)$\;
    \Indm
    $N \gets \sum_i \mathbf{1}_{\{M'^i=1\}}$\;
    $\mathcal{L} \gets S / N$\;
    \Return $\mathcal{L}$\;
\Indm
\end{algorithm}

\section{Hyperparameters and Detailed Settings}
\label{app:hyper}

\subsection{Training Hyperparameters and Settings}
Our training is divided into two settings: LoRA fine-tuning and full parameter fine-tuning.
LoRA fine-tuning is conducted on 4 H100 GPUs, while full parameter fine-tuning is conducted on 8 H100 GPUs.

For LoRA fine-tuning, experiments are conducted with a learning rate of $1\times 10^{-5}$ and a total of 20 training epochs.
We apply gradient accumulation with $4$ steps and used an effective per-device training batch size of $1$.
To reduce memory cost, we employ LoRA with rank $r=128$, scaling factor $\alpha=256$, and a dropout rate of $0.05$.
All hyperparameters, including model architecture, optimization, LoRA configuration, and training settings, are listed in Table~\ref{tab:hyperparams}.

\begin{table*}[htbp]
\centering
\begin{tabularx}{\textwidth}{@{}lX@{}}
\toprule
\textbf{Category} & \textbf{Hyperparameter = Value} \\
\midrule
\textbf{Optimization} & 
optimizer$\ =\ $adamw\_torch; weight\_decay$\ =0.1$; \\
& learning\_rate = 1e-5;  lr\_scheduler\_type = linear; \\
& adam\_beta1 = 0.9; adam\_beta2 = 0.999; adam\_epsilon = 1e-8; \\
& init\_fn$\ =\ $mitchell;
init\_std$\ =0.02$; max\_grad\_norm$\ =1$ \\
\midrule
\textbf{LoRA Configuration} & 
LoRA$\ =\ $\texttt{True}; $r=128$; lora\_alpha$\ =256$; lora\_dropout$\ =0.05$; \\
& target\_modules $=$ ["q\_proj", "k\_proj", "v\_proj"] \\
\midrule
\textbf{Training} &  num\_train\_epochs$\ =20$;  per\_device\_train\_batch\_size$\ =1$; \\ & gradient\_accumulation\_steps$\ =4$; precision = bf16; \\ &  max\_sequence\_length = 4096; seed $\ =42$ \\
\bottomrule
\end{tabularx}
\caption{Hyperparameters used in LoRA training.}
\label{tab:hyperparams}
\end{table*}

For full parameter fine-tuning, we train for 5 epochs with a learning rate of $2\times 10^{-5}$.
Other settings are detailed in Table~\ref{tab:hyperparams_lora}.

\begin{table*}[htbp]
\centering
\begin{tabularx}{\textwidth}{@{}lX@{}}
\toprule
\textbf{Category} & \textbf{Hyperparameter = Value} \\
\midrule
\textbf{Optimization} & 
optimizer$\ =\ $adamw\_torch; weight\_decay$\ =0$; \\
& learning\_rate = 2e-5;  lr\_scheduler\_type = cosine; \\
& adam\_beta1 = 0.9; adam\_beta2 = 0.999; adam\_epsilon = 1e-8; \\
& init\_fn$\ =\ $mitchell;
init\_std$\ =0.02$; max\_grad\_norm$\ =1$ \\
\midrule
\textbf{Training} &  num\_train\_epochs$\ =5$;  per\_device\_train\_batch\_size$\ =4$; \\ & gradient\_accumulation\_steps$\ =1$; precision = bf16; \\ &  max\_sequence\_length = 4096; seed $\ =42$ \\
\bottomrule
\end{tabularx}
\caption{Hyperparameters used in full parameter training.}
\label{tab:hyperparams_lora}
\end{table*}

\subsection{Evaluation Hyperparameters and Settings}
In the evaluation, we set the block length to 32. For the two smaller datasets, Sudoku and Countdown, the generation length is set to 512; for the two larger datasets, MATH500 and GSM8K, the generation length is set to 256. The number of diffusion steps is always set to half of the generation length, and we use the low-confidence remasking method. Regarding the specific partitioning of validation sets for these four datasets, we follow the settings in \texttt{wd1}~\citep{tang2025wd1weightedpolicyoptimization}.

\section{Motivation for Using Entropy-based Methods as the Weight Coefficient}
\label{app:mot}

In this section, we provide an intuitive perspective to justify the choice of \emph{entropy} as the coefficient~$\beta$, as this section is intended to provide motivation rather than a rigorous proof.

In brief, our approach of using entropy as an importance weight helps reduce the variance of the loss estimation. Specifically, prior work has demonstrated that importance sampling achieves minimum variance when the sampling proposal is proportional to the L2 norm of the gradient. Through reasonable approximation, we found that token-wise entropy is approximately proportional to the L2 norm of the gradient, which motivates our choice of this metric. A detailed discussion is provided below.

In~\cite{alain2016variancereductionsgddistributed}, the authors proved that importance sampling has minimum variance when
the sampling proposal is proportional to the $L_2$-norm of the gradient.

Recall that in a model, $p$ is generated by applying softmax to the output logits. More specifically, $$p_k=\frac{e^{z_k}}{\sum_je^{z_j}}$$.

Here $z$ is the logits generated by the model.

Let $c$ be the label and $p_c$ be the probability of this token.
Then we have $$\frac{\partial \log p_c}{\partial z_k}=\left\{
\begin{array}{ll}
-p_k & \text{if } k\neq c \\
1-p_c & \text{if } k =c
\end{array}
\right.$$

In the following discussion, we assume that the base model generates the correct token with relatively high probability and other tokens with relatively low probability. Therefore, $L_2$-norm of $\frac{\partial \log p_c}{\partial \mathbf{z}}$ can be approximated as $1-p_c$. 
Since $\frac{\partial \log p_c}{\partial \mathbf{z}}$ constitutes a key component of the loss gradient, theoretical insights suggest that selecting an importance sampling distribution that is positively correlated with this term might help reduce the variance of the loss, thus improving the stability and effectiveness of the training.
We now briefly analyze why entropy can serve this purpose.

The entropy function $H(\mathbf{p})$ admits the following lower bound:
\begin{align}
H(\mathbf{p}) &\ge - p_c \log p_c +(1 - p_c)\\
&(\log(n - 1)-\log(1-p_c)) \notag
\end{align}
where $n$ denotes the vocabulary size. Since the term $-p_c \log p_c$ satisfies
\begin{equation}
- p_c \log p_c < p_c(1 - p_c),
\end{equation}
it's negligible compared to $(1 - p_c)\log(n - 1)$. We may safely ignore it for approximation purposes.

Moreover, we empirically observe that the variation in incorrect predictions between different samples is relatively small.
Under this assumption, $H(\mathbf{p})$ is positively correlated with $(1 - p_c)$.

On the other hand, using entropy $H(\mathbf{p})$ is superior to directly using $(1 - p_c)$ in that it preserves richer information about the underlying distribution.
Due to the prevalence of synonyms and semantically similar tokens, the true probability distribution is unlikely to be unimodal.
Instead, it is more reasonable to expect the probability mass to be concentrated on a small subset of tokens.
In this respect, entropy $H(\mathbf{p})$ captures such structural information significantly better than $(1 - p_c)$.
In pratice, we employ a power function based on $\sqrt{H(\mathbf{p})}$ as the importance weighting function.

\section{SFT Role in Subsequent RL}
\label{app:rl}

\begin{figure}[h]
    \centering
    \includegraphics[width=0.6\columnwidth]{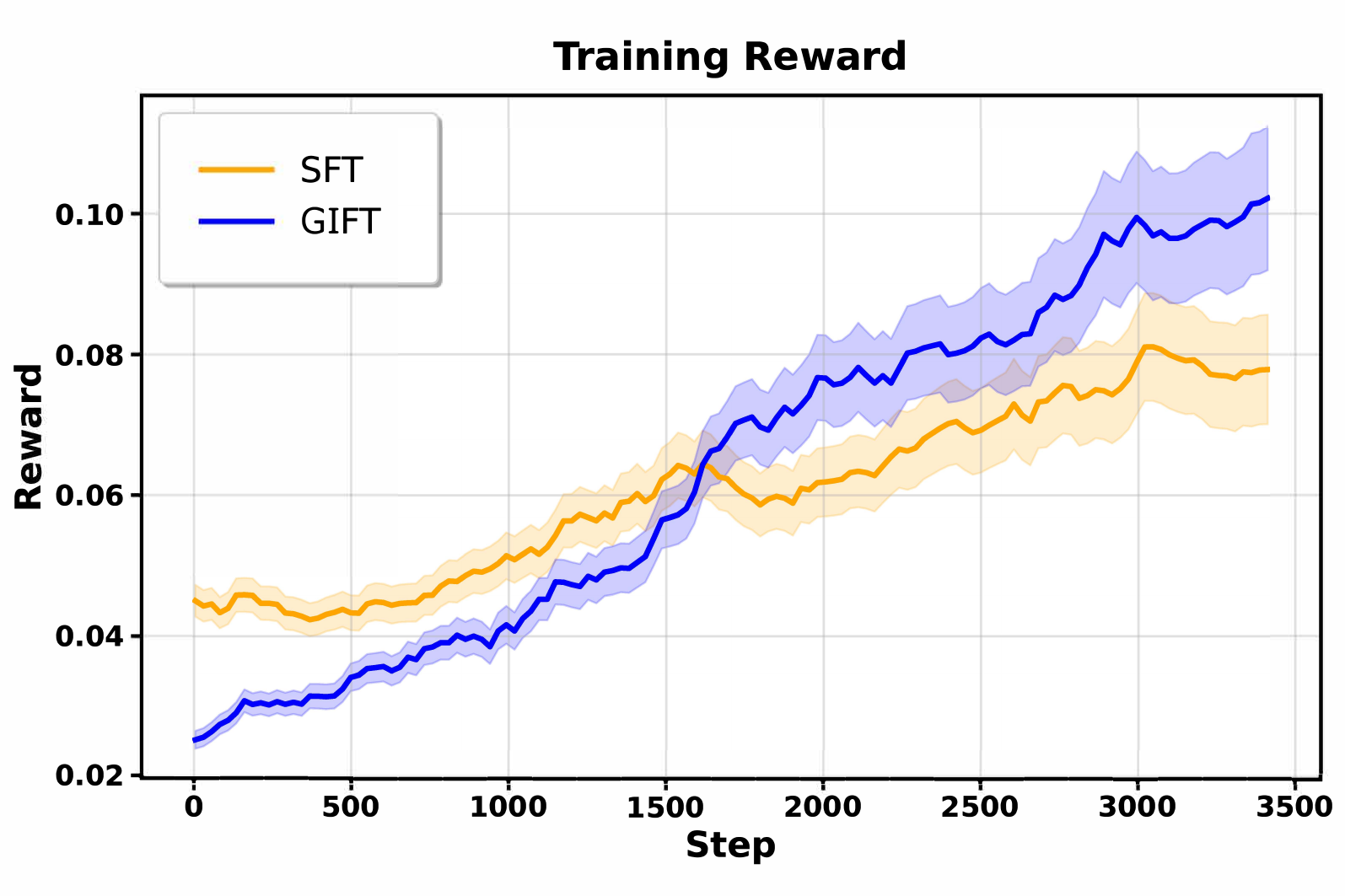} 
    \caption{Reward curves of models cold-started with \ours \ and SFT during subsequent reinforcement learning training. The curves are smoothed using a time-weighted EMA. As shown, the model initialized with \ours \ achieves higher rewards.  }
    \label{fig:rl}
\end{figure}

Reinforcement learning (RL) post-training has become a crucial stage in the training of large language models ~\citep{deepseekai2025deepseekr1incentivizingreasoningcapability}. Supervised fine-tuning (SFT) cold-start plays an important role in preparing the model for subsequent RL, as SFT establishes a reliable initial policy which can then be further refined by RL objectives. Therefore, it is essential to verify whether models trained with \ours \ can maintain their advantages during subsequent RL training. To this end, we fine-tune models on the s1K dataset using both \ours \ and SFT, and then further train them with RL on the Sudoku dataset for comparison.  

As shown in Figure~\ref{fig:rl}, models initialized with \ours \ achieve higher rewards throughout RL training. Moreover, as reported in Table~\ref{tab:rl}, the \ours-trained model achieves significantly better evaluation results. These findings demonstrate that models cold-started with \ours \ are effective in guiding subsequent RL learning.  

\begin{table*}[t]
\centering
\begin{tabular}{lccccc}
\toprule
 & \textbf{Base} &\textbf{ SFT }& \textbf{\ours} & \textbf{SFT+RL} & \textbf{\ours+RL} \\
\midrule
Sudoku & 5.5 & 4.6 & \textbf{7.8} & 8.9 & \textbf{13.3} \\
\bottomrule
\end{tabular}
\caption{Evaluation results of models cold-started with \ours{} and SFT during subsequent RL training.}
\label{tab:rl}
\end{table*}

\section{Empirical Necessity of Applying the Square Root of Entropy During Training}
\label{exp:ablation_entropy}

\begin{table}[H]
\centering
\begin{tabular*}{1\linewidth}{l@{\extracolsep{\fill}}cc}
\toprule
\textbf{Gradient Norm} & \textbf{Entropy} & \textbf{$\sqrt{\text{Entropy}}$}  \\ 
\midrule
Max  &  646.1 &  46.2  \\
Median & 0.47 & 0.43\\
Average &  2.02 &  1.03  \\
\bottomrule
\end{tabular*}
\caption{Comparison of gradient norms during training when using entropy versus square root entropy. Using entropy sometimes leads to extremely large gradient norms, resulting in instability, whereas using square root entropy effectively alleviates this issue.}
\label{tab:norm}
\end{table}

Below, we explain why we choose the square root of entropy rather than entropy as the metric. As shown in Table \ref{tab:norm}, we observed that directly using token weights based on raw entropy leads to large variations across tokens, which can significantly reduce training stability and even cause oscillations. To address this issue, we implement \ours \ using the square root of entropy as the weighting factor. This modification effectively reduces the scale of the gradient norms, resulting in more stable training and improved overall convergence behavior.

\section{Evaluation with Few-Shot Learning}
\label{app:few}
To further strengthen our empirical results and address the assessment of in-context learning capabilities, we include performance metrics for 5-shot evaluation. As shown in Table~\ref{tab:five_shot}, GIFT consistently demonstrates robust performance in the few-shot setting across different datasets.

\begin{table*}[h]
\centering
\begin{tabular*}{\textwidth}{l@{\extracolsep{\fill}}ccccc}
\toprule
\multirow{2}{*}{\textbf{Model}} & 
\textbf{Sudoku}    & \textbf{Countdown} & \textbf{GSM8K} & \textbf{MATH500} & \multirow{2}{*}{\textbf{Average}} \\
& \textit{5-shot} & \textit{5-shot}   & \textit{5-shot} & \textit{5-shot} & \\
\midrule
\multicolumn{6}{l}{\textbf{s1K Dataset, LoRA Tuning}} \\
\midrule
Instruct Model + SFT & $6.0 \pm 0.2$ & $23.2 \pm 1.0$ & \textbf{78.9 $\pm$ 0.3} & $31.8 \pm 0.6$ & $35.0 \pm 0.3$ \\
Instruct Model + \ours  & \textbf{7.9 $\pm$ 0.5} & \textbf{27.5 $\pm$ 1.6} & $76.7 \pm 0.2$ & \textbf{33.0 $\pm$ 0.7} & \textbf{36.3 $\pm$ 0.5} \\
\midrule
\multicolumn{6}{l}{\textbf{s1.1K Dataset, LoRA Tuning}} \\
\midrule
Instruct Model + SFT & \textbf{23.3 $\pm$ 0.3} & $21.7 \pm 1.2$ & \textbf{78.7 $\pm$ 0.1} & $30.0 \pm 0.5$ & $38.4 \pm 0.3$ \\
Instruct Model + \ours  & $20.1 \pm 0.8$ & \textbf{26.0 $\pm$ 1.2} & $78.2 \pm 0.1$ & \textbf{31.4 $\pm$ 0.6} & \textbf{38.9 $\pm$ 0.4} \\
\bottomrule
\end{tabular*}
\caption{5-shot evaluation results on Sudoku, Countdown, GSM8K, and MATH500 benchmarks. \ours \ consistently shows robust performance in few-shot scenarios.}
\label{tab:five_shot}
\end{table*}

\section{Discussion on the Rigor of the Theoretically Derived Loss}
\label{app:rigor}
Our theoretical derivation relies on certain necessary simplifying assumptions. Specifically, we assume a fixed matrix $Q$ during the derivation; however, in practical entropy-driven masking, $Q$ functions more as a heuristic variance-control scheme. 

To maintain theoretical rigor, we can also construct a training method that strictly adheres to our theoretical derivation by consistently using the initial, non-finetuned model as a reference model to calculate the importance weight $\beta$. In this setting, the transition matrix $Q$ remains fixed during training, making our theoretical derivation fully applicable. We denote this strictly theoretically-aligned version as GIFT*. We compare the performance of SFT and GIFT* using LoRA fine-tuning on the s1K dataset:

\begin{table*}[t]
\centering
\begin{tabular}{lccccc}
\toprule
 & \textbf{Sudoku} & \textbf{Countdown} & \textbf{GSM8K} & \textbf{MATH500} & \textbf{Average} \\
\midrule
SFT & $4.2 \pm 0.4$ & $21.1 \pm 1.4$ & \textbf{78.5 $\pm$ 0.2} & $32.6 \pm 0.9$ & $34.1 \pm 0.4$ \\
GIFT* & \textbf{5.4 $\pm$ 0.4} & \textbf{24.3 $\pm$ 0.3} & $77.2 \pm 0.3$ & \textbf{32.8 $\pm$ 0.4} & \textbf{34.9 $\pm$ 0.2} \\
\bottomrule
\end{tabular}
\caption{Comparison between SFT and the strictly theoretically-aligned GIFT* on the s1K dataset. Even with a fixed reference model, GIFT* consistently outperforms SFT across most tasks.}
\label{tab:gift_star_comparison}
\end{table*}

As shown in Table \ref{tab:gift_star_comparison}, the theoretically-aligned GIFT* method still achieves statistically significant improvements over SFT. One primary reason we choose to calculate the importance weight $\beta$ using the model being trained, rather than a fixed reference model is efficiency. This approach only requires loading the active model into GPU memory without the additional overhead of a reference model.

\section{Discussion on the Connection to Curriculum Learning}
\label{app:curriculum_learning}

Curriculum Learning ~\citep{soviany2022curriculumlearningsurvey} shares a similar high-level intuition with our approach, as both aim to improve training effectiveness by allocating different importance or tasks to the model at various stages of training. However, our method, \ours, differs from traditional Curriculum Learning in two fundamental aspects.

\paragraph{Granularity and Model-Adaptivity.} In Curriculum Learning, the difficulty of a problem is typically defined at the task or dataset level, and these difficulty measures are often static and independent of the specific model state. In contrast, the entropy-based weighting we employ reflects the difficulty of each individual token with respect to the current model state. This ensures that the training signal is inherently model-adaptive.
    
\paragraph{Training Paradigm.} Curriculum Learning generally follows an easy-to-hard paradigm, where the model is first exposed to simpler tasks before gradually transitioning to more complex ones. Conversely, \ours\ consistently focuses on the most challenging portions of the data (those that the current model finds most uncertain or confusing) to maximize learning efficiency.

%% file: custom.bib
@inproceedings{
nie2025largelanguagediffusionmodels,
title={Large Language Diffusion Models},
author={Shen Nie and Fengqi Zhu and Zebin You and Xiaolu Zhang and Jingyang Ou and Jun Hu and JUN ZHOU and Yankai Lin and Ji-Rong Wen and Chongxuan Li},
booktitle={ICLR 2025 Workshop on Deep Generative Model in Machine Learning: Theory, Principle and Efficacy},
year={2025},
url={https://openreview.net/forum?id=wzl61tIUj6}
}

@inproceedings{
ou2025absorbingdiscretediffusionsecretly,
title={Your Absorbing Discrete Diffusion Secretly Models the Conditional Distributions of Clean Data},
author={Jingyang Ou and Shen Nie and Kaiwen Xue and Fengqi Zhu and Jiacheng Sun and Zhenguo Li and Chongxuan Li},
booktitle={The Thirteenth International Conference on Learning Representations},
year={2025},
url={https://openreview.net/forum?id=sMyXP8Tanm}
}

@misc{tang2025wd1weightedpolicyoptimization,
      title={wd1: Weighted Policy Optimization for Reasoning in Diffusion Language Models}, 
      author={Xiaohang Tang and Rares Dolga and Sangwoong Yoon and Ilija Bogunovic},
      year={2025},
      eprint={2507.08838},
      archivePrefix={arXiv},
      primaryClass={cs.LG},
      url={https://arxiv.org/abs/2507.08838}, 
}

@inproceedings{
austin2023structureddenoisingdiffusionmodels,
title={Structured Denoising Diffusion Models in Discrete State-Spaces},
author={Jacob Austin and Daniel D. Johnson and Jonathan Ho and Daniel Tarlow and Rianne van den Berg},
booktitle={Advances in Neural Information Processing Systems},
editor={A. Beygelzimer and Y. Dauphin and P. Liang and J. Wortman Vaughan},
year={2021},
url={https://openreview.net/forum?id=h7-XixPCAL}
}

@inproceedings{
sahoo2024simpleeffectivemaskeddiffusion,
title={Simple and Effective Masked Diffusion Language Models},
author={Subham Sekhar Sahoo and Marianne Arriola and Aaron Gokaslan and Edgar Mariano Marroquin and Alexander M Rush and Yair Schiff and Justin T Chiu and Volodymyr Kuleshov},
booktitle={The Thirty-eighth Annual Conference on Neural Information Processing Systems},
year={2024},
url={https://openreview.net/forum?id=L4uaAR4ArM}
}

@inproceedings{
shi2025simplifiedgeneralizedmaskeddiffusion,
title={Simplified and Generalized Masked Diffusion for Discrete Data},
author={Jiaxin Shi and Kehang Han and Zhe Wang and Arnaud Doucet and Michalis Titsias},
booktitle={The Thirty-eighth Annual Conference on Neural Information Processing Systems},
year={2024},
url={https://openreview.net/forum?id=xcqSOfHt4g}
}

@misc{zhu2025llada15variancereducedpreference,
      title={LLaDA 1.5: Variance-Reduced Preference Optimization for Large Language Diffusion Models}, 
      author={Fengqi Zhu and Rongzhen Wang and Shen Nie and Xiaolu Zhang and Chunwei Wu and Jun Hu and Jun Zhou and Jianfei Chen and Yankai Lin and Ji-Rong Wen and Chongxuan Li},
      year={2025},
      eprint={2505.19223},
      archivePrefix={arXiv},
      primaryClass={cs.LG},
      url={https://arxiv.org/abs/2505.19223}, 
}

@article{huang2025reinforcingdiffusionchainlateral,
  publtype={informal},
  author={Zemin Huang and Zhiyang Chen and Zijun Wang and Tiancheng Li and Guo-Jun Qi},
  title={Reinforcing the Diffusion Chain of Lateral Thought with Diffusion Language Models},
  year={2025},
  month={May},
  cdate={1746057600000},
  journal={CoRR},
  volume={abs/2505.10446},
  url={https://doi.org/10.48550/arXiv.2505.10446}
}

@inproceedings{
haxholli2025efficientperplexityboundratio,
title={Efficient Perplexity Bound and Ratio Matching in Discrete Diffusion Language Models},
author={Etrit Haxholli and Yeti Z. Gurbuz and O{\u{g}}ul Can and Eli Waxman},
booktitle={The Thirteenth International Conference on Learning Representations},
year={2025},
url={https://openreview.net/forum?id=Mri9WIfxSm}
}

@inproceedings{he2022diffusionbertimprovinggenerativemasked,
  author={Zhengfu He and Tianxiang Sun and Qiong Tang and Kuanning Wang and Xuanjing Huang and Xipeng Qiu},
  title={DiffusionBERT: Improving Generative Masked Language Models with Diffusion Models},
  year={2023},
  cdate={1672531200000},
  pages={4521-4534},
  url={https://doi.org/10.18653/v1/2023.acl-long.248},
  booktitle={ACL (1)}
}

@misc{wang20258020rulehighentropyminority,
      title={Beyond the 80/20 Rule: High-Entropy Minority Tokens Drive Effective Reinforcement Learning for LLM Reasoning}, 
      author={Shenzhi Wang and Le Yu and Chang Gao and Chujie Zheng and Shixuan Liu and Rui Lu and Kai Dang and Xionghui Chen and Jianxin Yang and Zhenru Zhang and Yuqiong Liu and An Yang and Andrew Zhao and Yang Yue and Shiji Song and Bowen Yu and Gao Huang and Junyang Lin},
      year={2025},
      eprint={2506.01939},
      archivePrefix={arXiv},
      primaryClass={cs.CL},
      url={https://arxiv.org/abs/2506.01939}, 
}

@misc{wu2025generalizationsftreinforcementlearning,
      title={On the Generalization of SFT: A Reinforcement Learning Perspective with Reward Rectification}, 
      author={Yongliang Wu and Yizhou Zhou and Zhou Ziheng and Yingzhe Peng and Xinyu Ye and Xinting Hu and Wenbo Zhu and Lu Qi and Ming-Hsuan Yang and Xu Yang},
      year={2025},
      eprint={2508.05629},
      archivePrefix={arXiv},
      primaryClass={cs.LG},
      url={https://arxiv.org/abs/2508.05629}, 
}

@misc{cui2025entropymechanismreinforcementlearning,
      title={The Entropy Mechanism of Reinforcement Learning for Reasoning Language Models}, 
      author={Ganqu Cui and Yuchen Zhang and Jiacheng Chen and Lifan Yuan and Zhi Wang and Yuxin Zuo and Haozhan Li and Yuchen Fan and Huayu Chen and Weize Chen and Zhiyuan Liu and Hao Peng and Lei Bai and Wanli Ouyang and Yu Cheng and Bowen Zhou and Ning Ding},
      year={2025},
      eprint={2505.22617},
      archivePrefix={arXiv},
      primaryClass={cs.LG},
      url={https://arxiv.org/abs/2505.22617}, 
}

@article{lin2025rho1tokensneed,
  publtype={informal},
  author={Zhenghao Lin and Zhibin Gou and Yeyun Gong and Xiao Liu and Yelong Shen and Ruochen Xu and Chen Lin and Yujiu Yang and Jian Jiao and Nan Duan and Weizhu Chen},
  title={Rho-1: Not All Tokens Are What You Need},
  year={2024},
  cdate={1704067200000},
  journal={CoRR},
  volume={abs/2404.07965},
  url={https://doi.org/10.48550/arXiv.2404.07965}
}

@inproceedings{
luo2023reweightingtokenssimpleeffective,
title={Re-weighting Tokens: A Simple and Effective Active Learning Strategy for Named Entity Recognition},
author={Haocheng Luo and Wei Tan and Ngoc Dang Nguyen and Lan Du},
booktitle={The 2023 Conference on Empirical Methods in Natural Language Processing},
year={2023},
url={https://openreview.net/forum?id=CihCvXPiEG}
}

@misc{lou2024discretediffusionmodelingestimating,
      title={Discrete Diffusion Modeling by Estimating the Ratios of the Data Distribution}, 
      author={Aaron Lou and Chenlin Meng and Stefano Ermon},
      year={2024},
      eprint={2310.16834},
      archivePrefix={arXiv},
      primaryClass={stat.ML},
      url={https://arxiv.org/abs/2310.16834}, 
}

@misc{yang2025mmadamultimodallargediffusion,
      title={MMaDA: Multimodal Large Diffusion Language Models}, 
      author={Ling Yang and Ye Tian and Bowen Li and Xinchen Zhang and Ke Shen and Yunhai Tong and Mengdi Wang},
      year={2025},
      eprint={2505.15809},
      archivePrefix={arXiv},
      primaryClass={cs.CV},
      url={https://arxiv.org/abs/2505.15809}, 
}

@misc{schulman2017proximalpolicyoptimizationalgorithms,
      title={Proximal Policy Optimization Algorithms}, 
      author={John Schulman and Filip Wolski and Prafulla Dhariwal and Alec Radford and Oleg Klimov},
      year={2017},
      eprint={1707.06347},
      archivePrefix={arXiv},
      primaryClass={cs.LG},
      url={https://arxiv.org/abs/1707.06347}, 
}

@misc{cobbe2021trainingverifierssolvemath,
      title={Training Verifiers to Solve Math Word Problems}, 
      author={Karl Cobbe and Vineet Kosaraju and Mohammad Bavarian and Mark Chen and Heewoo Jun and Lukasz Kaiser and Matthias Plappert and Jerry Tworek and Jacob Hilton and Reiichiro Nakano and Christopher Hesse and John Schulman},
      year={2021},
      eprint={2110.14168},
      archivePrefix={arXiv},
      primaryClass={cs.LG},
      url={https://arxiv.org/abs/2110.14168}, 
}

@inproceedings{
lightman2023letsverifystepstep,
title={Let's Verify Step by Step},
author={Hunter Lightman and Vineet Kosaraju and Yuri Burda and Harrison Edwards and Bowen Baker and Teddy Lee and Jan Leike and John Schulman and Ilya Sutskever and Karl Cobbe},
booktitle={The Twelfth International Conference on Learning Representations},
year={2024},
url={https://openreview.net/forum?id=v8L0pN6EOi}
}

@misc{arel_sudoku,
  author       = {Arel},
  title        = {Arel's Sudoku Generator},
  howpublished = {\url{https://www.ocf.berkeley.edu/~arel/sudoku/main.html}},
  note         = {Accessed: 2025-04-08},
  year         = {2025},
}

@misc{tinyzero,
author       = {Jiayi Pan and Junjie Zhang and Xingyao Wang and Lifan Yuan and Hao Peng and Alane Suhr},
title        = {TinyZero},
howpublished = {https://github.com/Jiayi-Pan/TinyZero},
note         = {Accessed: 2025-01-24},
year         = {2025}
}

@inproceedings{
campbell2022continuoustimeframeworkdiscrete,
title={A Continuous Time Framework for Discrete Denoising Models},
author={Andrew Campbell and Joe Benton and Valentin De Bortoli and Tom Rainforth and George Deligiannidis and Arnaud Doucet},
booktitle={Advances in Neural Information Processing Systems},
editor={Alice H. Oh and Alekh Agarwal and Danielle Belgrave and Kyunghyun Cho},
year={2022},
url={https://openreview.net/forum?id=DmT862YAieY}
}

@book{anderson2012continuous,
  title={Continuous-time Markov chains: An applications-oriented approach},
  author={Anderson, William J},
  year={2012},
  publisher={Springer Science \& Business Media}
}

@misc{ye2025dream7bdiffusionlarge,
      title={Dream 7B: Diffusion Large Language Models}, 
      author={Jiacheng Ye and Zhihui Xie and Lin Zheng and Jiahui Gao and Zirui Wu and Xin Jiang and Zhenguo Li and Lingpeng Kong},
      year={2025},
      eprint={2508.15487},
      archivePrefix={arXiv},
      primaryClass={cs.CL},
      url={https://arxiv.org/abs/2508.15487}, 
}

@inproceedings{
muennighoff2025s,
title={s1: Simple test-time scaling},
author={Niklas Muennighoff and Zitong Yang and Weijia Shi and Xiang Lisa Li and Li Fei-Fei and Hannaneh Hajishirzi and Luke Zettlemoyer and Percy Liang and Emmanuel Candes and Tatsunori Hashimoto},
booktitle={Workshop on Reasoning and Planning for Large Language Models},
year={2025},
url={https://openreview.net/forum?id=LdH0vrgAHm}
}

@misc{openr1_mixture_of_thoughts_2025,
  title        = {Mixture-of-Thoughts},
  author       = {Open R1 HuggingFace},
  howpublished = {\url{https://huggingface.co/datasets/open-r1/Mixture-of-Thoughts}},
  year         = {2025},
}

@article{gpt1,
  title={Improving Language Understanding by Generative Pre-Training},
  author={Radford, Alec and Narasimhan, Karthik and Salimans, Tim and Sutskever, Ilya},
  journal={OpenAI technical report},
  year={2018},
  url={https://cdn.openai.com/research-covers/language-unsupervised/language_understanding_paper.pdf}
}

@article{gpt2,
  title={Language Models are Unsupervised Multitask Learners},
  author={Radford, Alec and Wu, Jeffrey and Child, Rewon and Luan, David and Amodei, Dario and Sutskever, Ilya},
  journal={OpenAI technical report},
  year={2019},
  url={https://cdn.openai.com/better-language-models/language_models_are_unsupervised_multitask_learners.pdf}
}

@article{gpt3,
  title={Language Models are Few-Shot Learners},
  author={Brown, Tom and Mann, Benjamin and Ryder, Nick and Subbiah, Melanie and others},
  journal={arXiv preprint arXiv:2005.14165},
  year={2020},
  url={https://arxiv.org/abs/2005.14165}
}

@inproceedings{
sun2023scorebased,
title={Score-based Continuous-time Discrete Diffusion Models},
author={Haoran Sun and Lijun Yu and Bo Dai and Dale Schuurmans and Hanjun Dai},
booktitle={The Eleventh International Conference on Learning Representations },
year={2023},
url={https://openreview.net/forum?id=BYWWwSY2G5s}
}

@inproceedings{
hu2022lora,
title={Lo{RA}: Low-Rank Adaptation of Large Language Models},
author={Edward J Hu and yelong shen and Phillip Wallis and Zeyuan Allen-Zhu and Yuanzhi Li and Shean Wang and Lu Wang and Weizhu Chen},
booktitle={International Conference on Learning Representations},
year={2022},
url={https://openreview.net/forum?id=nZeVKeeFYf9}
}

@misc{deepseekai2025deepseekr1incentivizingreasoningcapability,
      title={DeepSeek-R1: Incentivizing Reasoning Capability in LLMs via Reinforcement Learning}, 
      author={DeepSeek-AI},
      year={2025},
      eprint={2501.12948},
      archivePrefix={arXiv},
      primaryClass={cs.CL},
      url={https://arxiv.org/abs/2501.12948}, 
}

@misc{zhao2025d1scalingreasoningdiffusion,
      title={d1: Scaling Reasoning in Diffusion Large Language Models via Reinforcement Learning}, 
      author={Siyan Zhao and Devaansh Gupta and Qinqing Zheng and Aditya Grover},
      year={2025},
      eprint={2504.12216},
      archivePrefix={arXiv},
      primaryClass={cs.CL},
      url={https://arxiv.org/abs/2504.12216}, 
}

@misc{alain2016variancereductionsgddistributed,
      title={Variance Reduction in SGD by Distributed Importance Sampling}, 
      author={Guillaume Alain and Alex Lamb and Chinnadhurai Sankar and Aaron Courville and Yoshua Bengio},
      year={2016},
      eprint={1511.06481},
      archivePrefix={arXiv},
      primaryClass={stat.ML},
      url={https://arxiv.org/abs/1511.06481}, 
}

@misc{lambert2025tulu3pushingfrontiers,
      title={Tulu 3: Pushing Frontiers in Open Language Model Post-Training}, 
      author={Nathan Lambert and Jacob Morrison and Valentina Pyatkin and Shengyi Huang and Hamish Ivison and Faeze Brahman and Lester James V. Miranda and Alisa Liu and Nouha Dziri and Shane Lyu and Yuling Gu and Saumya Malik and Victoria Graf and Jena D. Hwang and Jiangjiang Yang and Ronan Le Bras and Oyvind Tafjord and Chris Wilhelm and Luca Soldaini and Noah A. Smith and Yizhong Wang and Pradeep Dasigi and Hannaneh Hajishirzi},
      year={2025},
      eprint={2411.15124},
      archivePrefix={arXiv},
      primaryClass={cs.CL},
      url={https://arxiv.org/abs/2411.15124}, 
}

@misc{dllm,
    author = {Zhanhui Zhou and Lingjie Chen and Hanghang Tong and Dawn Song},
    title = {dLLM: Simple Diffusion Language Modeling},
    year = {2025},
    publisher = {GitHub},
    journal = {GitHub repository},
    howpublished = {\url{https://github.com/ZHZisZZ/dllm}},
}

@misc{soviany2022curriculumlearningsurvey,
      title={Curriculum Learning: A Survey}, 
      author={Petru Soviany and Radu Tudor Ionescu and Paolo Rota and Nicu Sebe},
      year={2022},
      eprint={2101.10382},
      archivePrefix={arXiv},
      primaryClass={cs.LG},
      url={https://arxiv.org/abs/2101.10382}, 
}

@misc{bie2025llada20scalingdiffusionlanguage,
      title={LLaDA2.0: Scaling Up Diffusion Language Models to 100B}, 
      author={Tiwei Bie and Maosong Cao and Kun Chen and Lun Du and Mingliang Gong and Zhuochen Gong and Yanmei Gu and Jiaqi Hu and Zenan Huang and Zhenzhong Lan and Chengxi Li and Chongxuan Li and Jianguo Li and Zehuan Li and Huabin Liu and Lin Liu and Guoshan Lu and Xiaocheng Lu and Yuxin Ma and Jianfeng Tan and Lanning Wei and Ji-Rong Wen and Yipeng Xing and Xiaolu Zhang and Junbo Zhao and Da Zheng and Jun Zhou and Junlin Zhou and Zhanchao Zhou and Liwang Zhu and Yihong Zhuang},
      year={2025},
      eprint={2512.15745},
      archivePrefix={arXiv},
      primaryClass={cs.LG},
      url={https://arxiv.org/abs/2512.15745}, 
}
